\documentclass[a4paper,twoside,11pt,reqno]{amsart}
\usepackage{fullpage}
\usepackage[T1]{fontenc}
\usepackage[utf8]{inputenc}

\usepackage{hyperref}
\usepackage[slantedGreeks, partialup, noDcommand]{kpfonts}
\usepackage{graphicx}
\usepackage[mathscr]{euscript}

\hypersetup{ocgcolorlinks=true,allcolors=testc}
\hypersetup{
     colorlinks   = true,
     citecolor    = black
}
\hypersetup{linkcolor=black}
\hypersetup{urlcolor=black}

\newcommand{\eqdef}{\stackrel{\mathrm{def}}{=}}

\newtheorem{theorem}{Theorem}

\newtheorem{proposition}[theorem]{Proposition}
\newtheorem{corollary}[theorem]{Corollary}

\newtheorem*{definition*}{Definition}

\newcommand{\R}{\mathbb{R}} 

\newcommand{\N}{\mathbb{N}}

\newcommand{\e}{\varepsilon}

\usepackage{dutchcal}

\newcommand{\mb}{\mathbb}
\newcommand{\ms}{\mathscr}

\title{Learning low-degree functions from \\ a logarithmic number of random queries}

\author{Alexandros Eskenazis}
\address{(A.~E.) Trinity College and Department of Pure Mathematics and Mathematical Statistics\\
University of Cambridge, UK.}
\email{ae466@cam.ac.uk}

\author{Paata Ivanisvili}
\address{(P.~I.)
Department of Mathematics,
University of California, Irvine\\
Irvine, CA 92617, USA.}
\email{pivanisv@uci.edu} 

\thanks{A.~E. was supported by a Junior Research Fellowship from Trinity College, Cambridge. P.~I. was partially supported by the NSF grants DMS-2152346 and CAREER-DMS-2152401.}

\begin{document}

\maketitle
\vspace{-3mm}

\begin{abstract}
We prove that every bounded function $f:\{-1,1\}^n\to[-1,1]$ of degree at most $d$ can be learned with $L_2$-accuracy $\e$ and confidence $1-\delta$ from $\log(\tfrac{n}{\delta})\,\e^{-d-1} C^{d^{3/2}\sqrt{\log d}}$ random queries, where $C>1$ is a universal finite constant.
\end{abstract}

\bigskip

{\footnotesize

}


\section{Introduction}

Every function $f:\{-1,1\}^n\to\R$ admits a unique Fourier--Walsh expansion of the form
\begin{equation} \label{eq:walsh}
\forall \ x\in\{-1,1\}^n, \qquad f(x)= \sum_{S\subseteq\{1,\ldots,n\}} \hat{f}(S) w_S(x),
\end{equation}
where $w_S(x) = \prod_{i\in S} x_i$ and the Fourier coefficients $\hat{f}(S)$ are given by
\begin{equation}
\forall \ S\subseteq\{1,\ldots,n\}, \qquad \hat{f}(S) = \frac{1}{2^n} \sum_{y\in\{-1,1\}^n} f(y) w_S(y).
\end{equation}
We say that $f$ has degree at most $d\in\{1,\ldots,n\}$ if $\hat{f}(S)=0$ for every subset $S$ with $|S|>d$.


\subsection{Learning functions on the hypercube}

Let $\ms{C}$ be a class of functions $f:\{-1,1\}^n\to\R$ on the $n$-dimensional discrete hypercube. The problem of learning the class $\ms{C}$ can be described as follows: given a source of \emph{examples} $(x,f(x))$, where $x\in\{-1,1\}^n$, for an unknown function \mbox{$f\in\ms{C}$}, compute a \emph{hypothesis} function $h:\{-1,1\}^n\to\R$ which is a good approximation of $f$ up to a given error in some prescribed metric. In this paper we will be interested in the \emph{random query model} with $L_2$-error, in which we are given $N$ independent examples $(x,f(x))$, each chosen uniformly at random from the discrete hypercube $\{-1,1\}^n$, and we want to efficiently construct a (random) function $h:\{-1,1\}^n\to\R$ such that $\|h-f\|_{L_2}^2 < \e$ with probability at least $1-\delta$, where $\e,\delta\in(0,1)$ are given accuracy and confidence parameters. The goal is to construct a randomized algorithm\mbox{ which produces the hypothesis function $h$ from a minimal number $N$ of examples.}

The above very general problem has been studied for decades in computational learning theory and many results are known\footnote{We will by no means attempt to survey this (vast) field, so we refer the interested reader to the relevant chapters of O'Donnell's book \cite{O'Do14} and the references therein.}, primarily for various classes $\ms{C}$ of structured Boolean functions $f:\{-1,1\}^n\to\{-1,1\}$. Already since the late 1980s, researchers used the Fourier--Walsh expansion \eqref{eq:walsh} to design such learning algorithms (see the survey \cite{Man94}). Perhaps the most classical of these is the \emph{Low-Degree Algorithm} of Linial, Mansour and Nisan \cite{LMN93} who showed that for the class $\ms{C}_b^d$ of all \emph{bounded} functions $f:\{-1,1\}^n\to[-1,1]$ of degree at most $d$ there exists an algorithm which produces an $\e$-approximation of $f$ with probability at least $1-\delta$ using \mbox{$N=\tfrac{2n^d}{\e}\log(\tfrac{2n^d}{\delta})$} samples. In this generality, the $O_{\e,\delta,d}(n^d\log n)$ estimate of \cite{LMN93} was the state of the art until the recent work \cite{IRRRY21} of Iyer, Rao, Reis, Rothvoss and Yehudayoff who employed analytic techniques to derive new bounds on the $\ell_1$-size of the Fourier spectrum of bounded functions (see also Section \ref{sec:3}) and used these estimates to show that $N=O_{\e,\delta,d}(n^{d-1}\log n)$ examples suffice to learn $\ms{C}_{b}^d$. The goal of the present paper is to further improve this result and show that in fact $N=O_{\e,\delta,d}(\log n)$ samples suffice for this purpose.

\begin{theorem} \label{thm:main}
Fix $\e,\delta\in(0,1)$, $n \in\N$, $d\in\{1,\ldots,n\}$ and a bounded function $f:\{-1,1\}^n\to[-1,1]$ of degree at most $d$. If $N\in\N$ satisfies
\begin{equation}
N\geq \min\left\{\frac{\exp(Cd^{3/2}\sqrt{\log d})}{\e^{d+1}}, \frac{4dn^{d}}{\e}\right\} \log\left(\frac{n}{\delta}\right),
\end{equation}
where $C\in(0,\infty)$ is a large numerical constant, then $N$ uniformly random independent queries of pairs $(x,f(x))$, where $x\in\{-1,1\}^n$, suffice for the construction of a random function $h:\{-1,1\}^n \to \R$ satisfying the condition $\|h-f\|_{L_2}^2 <\e$ with probability at least $1-\delta$.
\end{theorem}


The proof of Theorem \ref{thm:main} relies on some important approximation theoretic estimates going back to the 1930s which we shall now describe (see also \cite{DS14}). To the best of our knowledge, these tools had\mbox{ not yet been exploited in the computational learning theory literature.}


\subsection{The Fourier growth of Walsh polynomials in $\ell_{\frac{2d}{d+1}}$}
Estimates for the growth of coefficients of polynomials as a function of their degree and their maximum on compact sets go back to the early days of approximation theory (see \cite{BE95}). A seminal result of this nature is Littlewood's celebrated $\tfrac{4}{3}$-inequality \cite{Lit30} for bilinear forms which was later generalized by Bohnenblust and Hille \cite{BH31} for multilinear forms on the torus $\mb{T}^n$ or the unit square $[-1,1]^n$. By means of polarization, one can use this multilinear estimate to derive an inequality for polynomials which reads as follows\footnote{For $\alpha=(\alpha_1,\ldots,\alpha_n)\in(\N\cup\{0\})^n$, we use the standard notations $|\alpha|=\alpha_1+\cdots+\alpha_n$ and $x^\alpha = x_1^{\alpha_1}\cdots x_n^{\alpha_n}$.}. For every $\mb{K}\in\{\mb{R},\mb{C}\}$ and $d\in\N$, there exists $B_d^\mb{K}\in(0,\infty)$ such that for every $n\in\N$ and every coefficients $c_\alpha\in\mb{K}$, where $\alpha\in(\N\cup\{0\})^n$ with $|\alpha|\leq d$, we have
\begin{equation} \label{eq:bh}
\left( \sum_{|\alpha|\leq d} |c_\alpha|^{\frac{2d}{d+1}}\right)^{\frac{d+1}{2d}} \leq B_d^\mb{K} \max\bigg\{\  \bigg| \sum_{|\alpha|\leq d} c_\alpha x^\alpha\bigg|: \ x\in\mb{K}^n \ \mbox{ with } \ \|x\|_{\ell_\infty^n(\mb{K})}\leq1 \bigg\}.
\end{equation}
Moreover, $\tfrac{2d}{d+1}$ is the smallest exponent for which the optimal constant in \eqref{eq:bh} is independent of the number of variables $n$ of the polynomial. The exact asymptotics of the constants $B_d^\R$ and $B_d^\mb{C}$ remain unknown, however it is known that there is a significant gap between $B_{d}^{\mathbb{R}}$ and $B_{d}^{\mathbb{C}}$, namely that
$\limsup_{d \to \infty} (B^{\mathbb{R}}_{d})^{1/d}=1+\sqrt{2}$ whereas $B_{d}^{\mathbb{C}} \leq C^{\sqrt{d \ln d}}$ for a finite constant $C>1$ (see \cite{DFOOS11,BPS14,DS14,CJMPS15,DMP19} for these and other important advances of the last decade). Restricting inequality \eqref{eq:bh} to real \emph{multilinear} polynomials, convexity shows that the maximum on the right-hand side is attained at a point $x\in\{-1,1\}^n$, which, in view of \eqref{eq:walsh}, makes \eqref{eq:bh} an estimate for the Fourier--Walsh growth of functions on the discrete hypercube. We shall denote by $B_d^{\{\pm 1\}}$ the corresponding optimal constant (first explicitly investigated by Blei in \cite[p.~175]{Ble01}), that is, the least constant such that for \mbox{every $n\in\N$ and every function $f:\{-1,1\}^n\to\R$ of degree at most $d$,}
\begin{equation} \label{eq:dmp}
\left( \sum_{S\subseteq\{1,\ldots,n\}} |\hat{f}(S)|^{\frac{2d}{d+1}} \right)^{\frac{d+1}{2d}} \leq B_d^{\{\pm 1\}} \ \|f\|_{L_\infty}.
\end{equation}
The best known quantitative result in this setting is due to Defant, Masty\l o and P\'erez \cite{DMP19} who showed that $B_d^{\{\pm 1\}} \leq \exp(\kappa \sqrt{d\log d})$ for a universal constant $\kappa\in(0,\infty)$. The main contribution of this work is the following theorem relating the growth of the constant $B_d^{\{\pm 1\}}$ and learning.

\begin{theorem} \label{thm:main2}
Fix $\e,\delta\in(0,1)$, $n\in\N$, $d\in\{1,\ldots,n\}$ and a bounded function $f:\{-1,1\}^n\to[-1,1]$ of degree at most $d$. If $N\in\N$ satisfies
\begin{equation}
N\geq \frac{e^8d^{2}}{\e^{d+1}} (B_{d}^{\{\pm 1\}})^{2d} \log\left(\frac{n}{\delta}\right),
\end{equation}
then given $N$ uniformly random independent queries of pairs $(x,f(x))$, where $x\in\{-1,1\}^n$, one can construct a random function $h:\{-1,1\}^n \to \R$ satisfying $\|h-f\|_{L_2}^2 <\e$ with probability at least $1-\delta$.
\end{theorem}


In Section \ref{sec:2} we will prove Theorem \ref{thm:main2} and use it to derive Theorem \ref{thm:main}. In Section \ref{sec:3} we will present some additional remarks on Boolean analysis and learning, in particular showing that the dependence on $n$ in Theorem \ref{thm:main} is optimal for $\delta\asymp\tfrac{1}{n}$. Moreover, we shall improve the recent bounds of \cite{IRRRY21} on the $\ell_1$-Fourier growth of bounded functions of low degree.

\subsection*{Acknowledgements}
We are very grateful to Assaf Naor for constructive feedback and to Lauritz Streck for useful discussions which led to Proposition \ref{prop:lb}.


\section{Proofs} \label{sec:2}

\begin{proof} [Proof of Theorem \ref{thm:main2}]
Fix a parameter $b\in(0,\infty)$ and denote by
\begin{equation} \label{eq:defN}
N_b \eqdef \left\lceil \frac{2}{b^2} \log\left(\frac{2}{\delta}\sum_{k=0}^d \binom{n}{k}\right) \right\rceil.
\end{equation}
Let $X_1,\ldots,X_{N_b}$ be independent random vectors, each uniformly distributed on $\{-1,1\}^n$. For a subset $S\subseteq\{1,\ldots,n\}$ with $|S|\leq d$ consider the empirical Walsh coefficient of $f$, given by
\begin{equation}
\alpha_S = \frac{1}{N_b} \sum_{j=1}^{N_b} f(X_j) w_S(X_j).
\end{equation}
As $\alpha_S$ is a sum of bounded i.i.d.~random variables and $\mb{E}[\alpha_S]=\hat{f}(S)$, the Chernoff bound gives
\begin{equation}
\forall \ S\subseteq\{1,\ldots,n\}, \qquad \mb{P}\big\{ |\alpha_S-\hat{f}(S)| > b\big\} \leq 2\exp(-N_bb^2/2).
\end{equation}
Therefore, using the union bound and taking into account that $f$ has degree at most $d$, we get
\begin{equation} \label{eq:defG}
\mb{P}\underbrace{\big\{ |\alpha_S-\hat{f}(S)| \leq b, \ \mbox{for every } S\subseteq\{1,\ldots,n\} \ \mbox{with} \ |S|\leq d\big\}}_{G_b} \geq 1-2\sum_{k=0}^d \binom{n}{k} \exp(-N_bb^2/2) \stackrel{\eqref{eq:defN}}{\geq} 1-\delta.
\end{equation}
Fix an additional parameter $a\in(b,\infty)$ and consider the random collection of sets given by
\begin{equation} \label{eq:defS}
\ms{S}_a \eqdef \big\{ S\subseteq\{1,\ldots,n\}: \ |\alpha_S| \geq a \big\}.
\end{equation}
Observe that if the event $G_b$ of equation \eqref{eq:defG} holds, then
\begin{equation} \label{eq:notinS}
\forall \ S\notin\ms{S}_a, \qquad |\hat{f}(S)| \leq |\alpha_S-\hat{f}(S)| + |\alpha_S| < a+b
\end{equation}
and
\begin{equation} \label{eq:inS}
\forall \ S\in\ms{S}_a, \qquad |\hat{f}(S)| \geq |\alpha_S| - |\alpha_S-\hat{f}(S)| \geq a-b.
\end{equation}
Finally, consider the random function $h_{a,b}:\{-1,1\}^n\to\R$ given by
\begin{equation}
\forall \ x\in\{-1,1\}^n, \qquad h_{a,b}(x) \eqdef \sum_{S\in\ms{S}_a} \alpha_S w_S(x).
\end{equation}
Combining \eqref{eq:inS} with inequality \eqref{eq:dmp}, we deduce that
\begin{equation} \label{eq:sizeS}
|\ms{S}_a| \stackrel{\eqref{eq:inS}}{\leq} (a-b)^{-\frac{2d}{d+1}} \sum_{S\in\ms{S}_a} |\hat{f}(S)|^{\frac{2d}{d+1}} \leq (a-b)^{-\frac{2d}{d+1}} \sum_{S\subseteq\{1,\ldots,n\}} |\hat{f}(S)|^{\frac{2d}{d+1}} \stackrel{\eqref{eq:dmp}}{\leq} (a-b)^{-\frac{2d}{d+1}} (B_d^{\{\pm 1\}})^\frac{2d}{d+1}.
\end{equation} 
Therefore, on the event $G_b$ we have
\begin{equation} \label{eq:first}
\begin{split}
\|h_{a,b}&-f\|_{L_2}^2  = \sum_{S\subseteq\{1,\ldots,n\}} \big|\hat{h}_{a,b}(S)-\hat{f}(S)\big|^2 = \sum_{S\in\ms{S}_a} |\alpha_S-\hat{f}(S)|^2 + \sum_{S\notin\ms{S}_a} |\hat{f}(S)|^2
\\ & \stackrel{\eqref{eq:notinS}}{<} |\ms{S}_a| b^2 + (a+b)^{\frac{2}{d+1}} \sum_{S\notin\ms{S}_a} |\hat{f}(S)|^{\frac{2d}{d+1}} \stackrel{\eqref{eq:dmp}\wedge\eqref{eq:sizeS}}{\leq}  (B_d^{\{ \pm 1\}})^\frac{2d}{d+1} \big( (a-b)^{-\frac{2d}{d+1}}b^2 + (a+b)^{\frac{2}{d+1}} \big).
\end{split}
\end{equation}
Choosing $a=b(1+\sqrt{d+1})$, we deduce that
\begin{equation}\label{uto2}
\|h_{b(1+\sqrt{d+1}),b}-f\|_{L_2}^2 < (B_d^{\{\pm 1\}})^\frac{2d}{d+1} b^{\frac{2}{d+1}} ((d+1)^{-\frac{d}{d+1}}+(2+\sqrt{d+1})^{\frac{2}{d+1}}).
\end{equation}
Next, we need the technical inequality 
\begin{align}\label{uto1}
(d+1)^{-\frac{d}{d+1}}+(2+\sqrt{d+1})^{\frac{2}{d+1}} \leq (e^{4} (d+1))^{\frac{1}{d+1}} \quad \text{for all} \quad d\geq 1. 
\end{align}
Rearranging the terms, it suffices to show that $(2+\sqrt{d+1})^{\frac{2}{d+1}} \leq (d+1)^{\frac{1}{d+1}}\big(e^{\frac{4}{d+1}}-\frac{1}{d+1}\big)$, which is equivalent to $\big(\frac{2}{\sqrt{d+1}} +1\big)^{\frac{2}{d+1}} \leq e^{\frac{4}{d+1}}-\frac{1}{d+1}$. We have 
\begin{equation}
\left(\frac{2}{\sqrt{d+1}} +1\right)^{\frac{2}{d+1}} \leq \left(\sqrt{2}+1 \right)^{\frac{2}{d+1}} \stackrel{(*)}{\leq} 1+\frac{3}{d+1}\leq e^{\frac{4}{d+1}}-\frac{1}{d+1},
\end{equation}
where inequality $(*)$ holds because the left hand side is convex in the variable $\lambda \eqdef\frac{2}{d+1}$ whereas the right hand side is linear and since $(*)$ holds at the endpoints $\lambda=0, 1$.

Combining (\ref{uto2}) and (\ref{uto1}) we see that $\|h_{b(1+\sqrt{d+1}),b}-f\|_{L_2}^2 <\varepsilon$ holds for $b^{2}\leq  e^{-5} d^{-1} \varepsilon^{d+1} (B_d^{\{\pm 1\}})^{-2d}$. Plugging this choice of $b$ in \eqref{eq:defN} shows that given $N$ random queries, where
\begin{equation}
N=\left\lceil \frac{e^6 d (B_d^{\{\pm 1\}})^{2d}}{\e^{d+1}} \log\left(\frac{2}{\delta}\sum_{k=0}^d \binom{n}{k}\right)\right\rceil,
\end{equation}
the random function $h_{b(1+\sqrt{d+1}),b}$ satisfies $\|h_{b(1+\sqrt{d+1}),b}-f\|_{L_2}^2<\e$ with probability at least $1-\delta$ and the conclusion of the theorem follows from elementary estimates, such as 
\begin{equation*}
\sum_{k=0}^{d}\binom{n}{k}\leq  \sum_{k=0}^{d} \frac{n^{k}}{k!} =\sum_{k=0}^{d} \frac{d^{k}}{k!} \left(\frac{n}{d}\right)^{k} \leq \left(\frac{en}{d}\right)^{d}. \qedhere
\end{equation*}
\end{proof}

Theorem \ref{thm:main} is a straightforward consequence of Theorem \ref{thm:main2}. 

\begin{proof} [Proof of Theorem \ref{thm:main}]
Theorem \ref{thm:main2} combined with the bound $B_d^{\{\pm 1\}}\leq \exp(\kappa\sqrt{d\log d})$ of \cite{DMP19} imply the conclusion of Theorem \ref{thm:main} for $\e\geq\tfrac{\exp(C\sqrt{d\log d})}{n}$, where $C\in(0,\infty)$ is a large universal constant. The case $\e<\tfrac{\exp(C\sqrt{d\log d})}{n}$ follows from the Low-Degree Algorithm of \cite{LMN93}.
\end{proof}

\section{Concluding remarks} \label{sec:3}

We conclude with a few additional remarks on the spectrum of bounded functions defined on the hypercube and corresponding learning algorithms. For a function $f:\{-1,1\}^n\to\R$, its Rademacher projection on level $\ell\in\{1,\ldots,n\}$ is defined as
\begin{equation}
\forall \ x\in\{-1,1\}^n, \qquad \mathrm{Rad}_\ell f(x) = \sum_{\substack{S\subseteq\{1,\ldots,n\} \\ |S|=\ell}} \hat{f}(S) w_S(x).
\end{equation}

\smallskip

\noindent {\bf 1.} The first main theorem of \cite{IRRRY21} asserts that if $f:\{-1,1\}^n\to\R$ is a function of degree $d$, then 
\begin{equation} \label{eq:Radinf}
\forall \ \ell\in\{1,\ldots,d\}, \qquad \big\| \mathrm{Rad}_\ell f\big\|_{L_\infty} \leq \begin{cases} \frac{|T_d^{(\ell)}(0)|}{\ell!}\cdot \|f\|_{L_\infty}, & \mbox{if } (d-\ell) \mbox{ is even} \\ \frac{|T_{d-1}^{(\ell)}(0)|}{\ell!}\cdot \|f\|_{L_\infty}, & \mbox{if } (d-\ell) \mbox{ is odd} \end{cases},
\end{equation}
where $T_d(t)$ is the $d$-th Chebyshev polynomial of the first kind, that is, the unique real polynomial of degree $d$ such that $\cos(d\theta) = T_d(\cos\theta)$ for every $\theta\in\R$. Moreover, Iyer, Rao, Reis, Rothvoss and Yehudayoff observed in \cite[Proposition~2]{IRRRY21} that this estimate is asymptotically sharp. We present a simple proof of their inequality \eqref{eq:Radinf} (see also \cite{EI20} for related arguments).

\begin{proof} [Proof of \eqref{eq:Radinf}]
For any $f :\{-1,1\}^{n} \to \mathbb{R}$ consider its harmonic extension on $[-1,1]^{n}$,  
\begin{equation}
\forall \ (x_1,\ldots,x_n)\in[-1,1]^n, \qquad \tilde{f}(x_{1}, \ldots, x_{n}) = \sum_{S \subseteq \{1 ,\ldots, n\}}\hat{f}(S) \prod_{j\in S}x_{j}.
\end{equation}
By convexity $\| \tilde{f}\|_{L^{\infty}([-1,1]^{n})}=\|f\|_{L^{\infty}(\{-1,1\}^{n})}$. In particular, the restriction of $\tilde{f}$ on the ray $t (x_{1}, \ldots, x_{n})$,  $t \in [-1,1]$, i.e.
\begin{equation}
\forall \ t\in\R, \qquad h_x(t)\eqdef \sum_{S\subseteq\{1,\ldots,n\}} \hat{f}(S)  w_S(x) t^{|S|}
\end{equation}
satisfies  $\max_{t\in[-1,1]} |h_x(t)| \leq \|f\|_{L_\infty}$ for all $(x_{1}, \ldots, x_{n}) \in \{-1,1\}^{n}$. Therefore, since $\mathrm{deg}h_x\leq d$, a classical inequality of Markov (see e.g.~\cite[p.~248]{BE95}) gives
\begin{equation}
\big|\mathrm{Rad}_\ell f(x)\big| = \frac{|h_x^{(\ell)}(0)|}{\ell!} \leq \begin{cases} \frac{|T_d^{(\ell)}(0)|}{\ell!}\cdot \|f\|_{L_\infty}, & \mbox{if } (d-\ell) \mbox{ is even} \\ \frac{|T_{d-1}^{(\ell)}(0)|}{\ell!}\cdot \|f\|_{L_\infty}, & \mbox{if } (d-\ell) \mbox{ is odd} \end{cases}
\end{equation}
and \eqref{eq:Radinf} follows by taking a maximum over all $x\in\{-1,1\}^n$.
\end{proof}

In particular, as observed in \cite{IRRRY21}, inequality \eqref{eq:Radinf} implies that if $f$ has degree at most $d$ then
\begin{equation} \label{eq:weakRad}
\forall \ \ell\in\{1,\ldots,d\}, \qquad \big\| \mathrm{Rad}_\ell f\big\|_{L_\infty} \leq \frac{d^\ell}{\ell!} \cdot \|f\|_{L_\infty}.
\end{equation}

\smallskip

\noindent {\bf 2.}  The second main theorem of \cite{IRRRY21} asserts that if $f:\{-1,1\}^n\to[-1,1]$ is a bounded function of degree at most $d$, then for every $\ell\in\{1,\ldots,d\}$ we have
\begin{equation}
\sum_{S\subseteq\{1,\ldots,n\}} |\widehat{\mathrm{Rad}_\ell f}(S)| = \sum_{\substack{S\subseteq\{1,\ldots,n\} \\ |S|=\ell}} |\hat{f}(S)| \leq n^{\frac{\ell-1}{2}} d^\ell e^{\binom{\ell+1}{2}}.
\end{equation}
The Bohnenblust--Hille-type inequality of \cite{DMP19} implies the following improved bound.

\begin{corollary}
Let $n\in\N$ and $d\in\{1,\ldots,n\}$. Then, every bounded function  $f:\{-1,1\}^n\to[-1,1]$ of degree at most $d$ satisfies
\begin{equation} \label{eq:cor3}
\forall \ \ell\in\{1,\ldots,d\}, \qquad  \sum_{\substack{S\subseteq\{1,\ldots,n\} \\ |S|=\ell}} |\hat{f}(S)| \leq \binom{n}{\ell} ^{\frac{\ell-1}{2\ell}} e^{\kappa\sqrt{\ell\log \ell}} \frac{d^\ell}{\ell!} \leq n^{\frac{\ell-1}{2}} d^\ell \ell^{-c\ell},
\end{equation}
for some universal constant $c\in(0,1)$.
\end{corollary}

\begin{proof}
Combining H\"older's inequality with the estimate of \cite{DMP19} and \eqref{eq:weakRad} we get
\begin{equation}
\begin{split}
\sum_{\substack{S\subseteq\{1,\ldots,n\} \\ |S|=\ell}} |\hat{f}&(S)|  \leq  \binom{n}{\ell}^{\frac{\ell-1}{2\ell}} 
\Bigg( \sum_{S\subseteq\{1,\ldots,n\}} |\widehat{\mathrm{Rad}_\ell f}(S)|^{\frac{2\ell}{\ell+1}} \Bigg)^{\frac{\ell+1}{2\ell}}  \\ & \stackrel{}{\leq} \binom{n}{\ell} ^{\frac{\ell-1}{2\ell}} \exp(\kappa\sqrt{\ell\log \ell}) \big\| \mathrm{Rad}_\ell f\big\|_{L_\infty} 
 \stackrel{\eqref{eq:weakRad}}{\leq} \binom{n}{\ell} ^{\frac{\ell-1}{2\ell}} \exp(\kappa\sqrt{\ell\log \ell}) \frac{d^\ell}{\ell!} .
\end{split}
\end{equation}
The last inequality of \eqref{eq:cor3} follows from \eqref{eq:Radinf} and the elementary bound $\binom{n}{\ell} \leq \big(\tfrac{ne}{\ell}\big)^\ell$.
\end{proof}

We refer to the recent work \cite{BIJLSV21} for a systematic study of inequalities relating the Fourier growth with various well-studied properties of Boolean functions.

\medskip

\noindent {\bf 3.} It is straightforward to observe (see also \cite[Proposition~3.31]{O'Do14}) that if $f:\{-1,1\}^n\to\{-1,1\}$ is a Boolean function and $h:\{-1,1\}^n\to\R$ is an arbitrary function, then
\begin{equation}
\big\|\mathrm{sign}(h) - f\big\|_{L_2}^2 = 4\mb{P}\{ \mathrm{sign}(h)\neq f\} \leq 4\mb{P}\{|h-f|\geq1\} \leq 4\|h-f\|_{L_2}^2,
\end{equation}
where we define $\mathrm{sign}(0)$ as $\pm1$ arbitrarily. Therefore, applying Theorem \ref{thm:main} to a Boolean function, the above algorithm produces a \emph{Boolean} function $\tilde{h}=\mathrm{sign}(h)$ which is a $4\e$-approximation of $f$.

\medskip

\noindent {\bf 4.} In Theorem \ref{thm:main} we showed that bounded functions $f:\{-1,1\}^n\to[-1,1]$ of degree at most $d$ can be learned with accuracy at most $\varepsilon$ and confidence at least $1-\delta$ from $N=O_{\e,d}\big(\log(n/\delta)\big)$ random queries. We will now show that this estimate is sharp for small enough values of $\delta$.

\begin{proposition} \label{prop:lb}
Suppose that bounded linear functions $\ell:\{-1,1\}^n\to[-1,1]$ can be learned with accuracy at most $\tfrac{1}{2}$ and confidence at least $1-\tfrac{1}{2n}$ from $N$ random queries. Then $N> \log_2 n$.
\end{proposition}

\begin{proof}
By the assumption, for any input $(X_1, y_1), \ldots, (X_N, y_N) \in \{-1,1\}^n\times [-1,1]$, there exists a function $h_{(X_1,y_1),\ldots,(X_N,y_N)}:\{-1,1\}^n\to\R$ such that if $X_1,\ldots,X_N$ are chosen independently and uniformly from $\{-1,1\}^n$ and there exists a linear function $\ell:\{-1,1\}^n\to[-1,1]$ such that $y_j = \ell(X_j)$ for every $j\in\{1,\ldots,N\}$, then $\mb{P}(\Omega_\ell)>1-\tfrac{1}{2n}$, where $\Omega_\ell$ is the event
\begin{equation} \label{eq:lb1}
\Omega_\ell\eqdef \Big\{\mb{E}\big( h_{(X_1,\ell(X_1)),\ldots,(X_N,\ell(X_N))} - \ell \big)^2 < \frac{1}{2}\Big\}.
\end{equation}
Let $X_j=(X_j(1),\ldots,X_j(n))$ for $j\in\{1,\ldots,N\}$ and consider the event
\begin{equation} 
\ms{W} = \big\{X_j(1) = X_j(2), \ \forall \ j\in\{1,\ldots,N\} \big\}.
\end{equation}
By the independence of the samples, we have $\mb{P}(\ms{W}) = \tfrac{1}{2^N}$. Therefore, if $N\leq \log_2n$ and we consider the linear functions $r_i:\{-1,1\}^n\to\{-1,1\}$ given by $r_i(x)=x_i$, then
\begin{equation}
\mb{P}(\Omega_{r_1}\cap\Omega_{r_2}) > 1-\frac{1}{n} \geq 1-\frac{1}{2^N}=1-\mb{P}(\ms{W}),
\end{equation}
which implies that $\Omega_{r_1}\cap\Omega_{r_2}\cap\ms{W}\neq\emptyset$. Choosing $X_1,\ldots,X_N$ from this event and denoting by $h=h_{(X_1,X_1(1)),\ldots,(X_N,X_N(1))} = h_{(X_1,X_1(2)),\ldots,(X_N,X_N(2))}$, we deduce from the triangle inequality that
\begin{equation}
2=\mb{E}(r_1-r_2)^2\leq 2\mb{E}(h-r_1)^2 + 2 \mb{E}(h-r_2)^2 \stackrel{\eqref{eq:lb1}}{<} 2
\end{equation} 
which is clearly a contradiction. Therefore $N>\log_2n$.
\end{proof}

\bibliographystyle{siam}
\bibliography{BH-learning}

\end{document}